\documentclass[11pt]{article}
\usepackage[margin=1in]{geometry}
\usepackage{amsmath,amssymb,amsthm}
\usepackage{booktabs}
\usepackage{graphicx}
\usepackage[hidelinks]{hyperref}
\usepackage{microtype}

\newtheorem{proposition}{Proposition}
\newtheorem{corollary}{Corollary}
\newcommand{\sgn}{\operatorname{sign}}
\newcommand{\diag}{\operatorname{diag}}

\title{ExTernD: Expanded-Rank Ternary Decomposition\\
\large Ternary LLM PTQ with Accuracy Approaching Any Quantization Level}
\author{Chethan Reddy G.P.\\ \texttt{chethanreddygp@proton.me}}
\date{July 2026 (early preprint)}

\begin{document}
\maketitle

\begin{abstract}
We introduce \textbf{ExTernD} (Expanded-rank Ternary Decomposition), a post-training factorization of each LLM weight matrix $A \in \mathbb{R}^{m\times n}$ into $A \approx B\,\diag(D)\,C$ with ternary factors $B \in \{-1,0,+1\}^{m\times k}$, $C \in \{-1,0,+1\}^{k\times n}$ and a real scale vector $D \in \mathbb{R}^k$. The \emph{inner rank} $k = \mu\min(m,n)$ is deliberately expanded beyond full rank ($\mu > 1$), so that components past full rank correct the quantization error of earlier ones. We prove the residual decreases monotonically in $k$ and can be driven below any $\varepsilon > 0$: ExTernD approaches bf16 accuracy arbitrarily closely, which no ternary scheme with a fixed plane count can do. Memory and compute scale continuously with $\mu$, and factor sparsity continuously with a threshold $\tau$, so an accuracy target is hit exactly rather than rounded to the next bit-width. ExTernD matches Q4\_K's per-matrix accuracy at 5.2--5.5 effective bpw (5.1--5.5 with importance weighting) on Gemma-4-E2B and Qwen3.5-4B, and a full Qwen3.5-4B conversion at $\mu = 3$ reaches 10.10 wikitext-2 perplexity against 9.78 for bf16 ($+3.2\%$), placing it near the Q4\_K/Q5\_K accuracy band at ${\sim}5.7$ effective bpw.
\end{abstract}

\section{Introduction}

Ternary weights are attractive: matrix multiplication degenerates into additions and subtractions, and storage approaches $\log_2 3 \approx 1.58$ bits per weight. BitNet \cite{bitnet,bitnet158} shows that 1-bit and 1.58-bit LLMs can be \emph{trained from scratch}; ternary weight networks \cite{twn} and trained ternary quantization \cite{ttq} established the format earlier for vision models. But a single ternary plane carries at most 1.58 bits of information per weight, so \emph{post-training} projection of a full-precision matrix onto one ternary plane loses too much: reasoning collapses \cite{ptqtp}.

PTQTP \cite{ptqtp} takes the important step of using \emph{two} trit-planes, a superposition of two ternary matrices with row-wise scales, and shows this restores much of the lost expressiveness without retraining. But two planes is still a fixed capacity budget with a ceiling it cannot get past, and there is nothing special about two, or about any fixed number of planes. ExTernD replaces a fixed count of full-size planes with a \emph{factored} representation whose inner dimension $k$ is a free parameter: $A \approx B\,\diag(D)\,C$ with both factors ternary. Setting $k = \mu\min(m,n)$ (typically $\mu = 2$--$3.5$, but unbounded) gives the representation as much capacity as needed; each additional component fits the residual left by all previous ones, so capacity is spent exactly where the ternary constraint hurt and accuracy is not capped (Sec.~\ref{sec:core}, proof in Appendix~\ref{app:proof}). All of this is strictly post-training: the factorization sees one weight matrix at a time and never a gradient.

Mainstream post-training quantization (GPTQ \cite{gptq}, AWQ \cite{awq}, llama.cpp k-quants) sits at a small number of discrete accuracy/cost points: 3, 4, 5 bits per weight. Because $k$ is a real-valued rank rather than a bit-width, ExTernD's accuracy/cost trade-off is instead continuous: the inner-rank multiplier $\mu$ (storage and compute both scale linearly with it, Sec.~\ref{sec:cost}) and a sparsity threshold $\tau$ (which sets the fraction of zeros in $B$ and $C$, e.g.\ $\tau{=}0.7 \to {\sim}41\%$ zeros, $\tau{=}1.0 \to {\sim}57\%$, $\tau{=}2.0 \to {\sim}87\%$) can each be set per matrix to any real value, letting a target accuracy be hit exactly rather than rounded up to the next bit-width.

Contributions: (i) ExTernD and a greedy ALS algorithm for it, with a proof (Appendix~\ref{app:proof}) that its error falls below any $\varepsilon > 0$ as $\mu$ grows, making it \textbf{to our knowledge the only ternary quantization that can provably reach any target accuracy}, up to bf16, and hence beat any fixed-plane-count method at large enough $\mu$; (ii) a batched block-ALS GPU algorithm with decorrelated targets, matching sequential quality at ${\sim}16\times$ speed; (iii) an importance-weighted variant using llama.cpp imatrix statistics, including a non-obvious correction to the V-step; (iv) an empirical study on Gemma-4-E2B, Qwen3.5-4B and Granite-4.0-h-tiny, including a full end-to-end model conversion.

\section{Method}

\subsection{Expanded inner rank}
\label{sec:core}

Given $A \in \mathbb{R}^{m\times n}$, we seek
\begin{equation}
A \;\approx\; \hat A \;=\; B\,\diag(D)\,C, \qquad B \in \{-1,0,1\}^{m\times k},\; C \in \{-1,0,1\}^{k\times n},\; D \in \mathbb{R}^k .
\end{equation}
Quality is reported as \textbf{energy preserved}, $E = 1 - \|A - \hat A\|_F^2 / \|A\|_F^2$.

With real factors, $k = \min(m,n)$ suffices (SVD). With ternary factors each rank-one component $d_i\, b_i c_i^\top$ is a crude, quantized object; at $k = \min(m,n)$ the representation hits an information ceiling well short of full precision. The core idea is to keep adding components \emph{past} full rank: component $i$ is fitted to the residual $R = A - \sum_{j<i} d_j b_j c_j^\top$, i.e.\ to the accumulated quantization error of its predecessors. The multiplier $\mu = k/\min(m,n)$ becomes the accuracy dial. Crucially this only works through \emph{sequential deflation}: solving for all $k > \min(m,n)$ components jointly is underdetermined and fails completely (0--50\% energy in our ablation, vs.\ 96--100\% with deflation).

Because $\mu$ is unbounded and the residual is monotone in it (Appendix~\ref{app:proof}), the accuracy ceiling that a fixed plane count imposes does not exist here: any target accuracy is reachable by construction rather than by luck of the fit.

\subsection{Ternarization operator and continuous sparsity}

For $u \in \mathbb{R}^m$, the adaptive mean-threshold ternarization with scale $\tau$ is
\begin{equation}
T_\tau(u)_i = \sgn(u_i)\cdot \mathbf{1}\!\Big[\,|u_i| > \tau \cdot \tfrac1m \textstyle\sum_j |u_j|\,\Big],
\end{equation}
keeping the single largest entry if the result would be all-zero. $\tau$ controls factor sparsity continuously and almost matrix-independently. The optimal ternary projection $\arg\min_{t,\alpha} \|u - \alpha t\|_2$ has a closed form (sort $|u|$, keep the top-$s^\star$ prefix maximizing $\frac{1}{\sqrt s}\sum_{i\le s}|u|_{(i)}$) but empirically ties $T_{0.7}$, so we use the cheaper threshold form.

\subsection{Greedy sequential ALS}
\label{sec:seq}

With residual $R \leftarrow A$, for $i = 1,\dots,k$:
\begin{enumerate}
\item \textbf{Alternating ternary fit} (15 iterations): $v \leftarrow T_\tau(R^\top u)$, \; $u \leftarrow T_\tau(R v)$.
\item \textbf{Optimal scale}: $d_i = \dfrac{u^\top R v}{\|u\|_2^2\,\|v\|_2^2}$, the least-squares minimizer of $\|R - d\,uv^\top\|_F$.
\item \textbf{Deflate}: $R \leftarrow R - d_i\,uv^\top$; store $B_{:,i} = u$, $C_{i,:} = v^\top$.
\end{enumerate}
Step 2 guarantees $\|R\|_F^2$ never increases (Appendix~\ref{app:proof}), which is what makes $\mu$ a well-behaved dial: more components monotonically means less error, all the way to exact recovery.

\subsection{Batched block ALS (GPU)}
\label{sec:batched}

Extracting components in blocks of $b$ (default 256, capped at $\lfloor\min(m,n)/8\rfloor$) turns per-component vector work into block matmuls. The naive batching, ternarizing $R^\top U$ column-wise, collapses (99.5\% $\to$ 50\% energy at $b{=}256$): every column chases the same dominant residual direction. The fix is to ternarize the \emph{jointly decorrelated least-squares target}:
\begin{equation}
V \leftarrow T_\tau\!\Big(\big[(U^\top U + \epsilon I)^{-1} U^\top R\big]^\top\Big), \qquad
U \leftarrow T_\tau\!\Big(R\,V\,(V^\top V + \epsilon I)^{-1}\Big),
\end{equation}
so each component is fitted against the residual minus its block-mates' contributions (implicit within-block deflation; an exact block solve in place of Gauss--Seidel). After the alternating iterations, all $b$ scales are solved jointly and exactly:
\begin{equation}
\big[(U^\top U)\odot(V^\top V)\big]\, d = \diag(U^\top R V),
\end{equation}
then $R \leftarrow R - U\diag(d)V^\top$. Deflation \emph{across} blocks stays fully sequential. Block width must respect $b \lesssim \min(m,n)/8$ or the ternary Gram matrices become ill-conditioned.

\textbf{Refinement sweeps.} After extraction, revisiting each block (add its contribution back to $R$, re-run the block fit, re-deflate) improves energy monotonically: ${\approx}{+}0.4$ points on hard matrices after 10 sweeps, and a 4--8\% reduction in the multiplier needed for fixed energy (larger at high $\tau$).

\subsection{Importance-weighted ALS}
\label{sec:imatrix}

Frobenius error weights all input channels equally; activations do not. Given per-input-channel second moments $h$ from a llama.cpp importance matrix (imatrix), we minimize $\sum_{ij} \tilde h_j\,(A_{ij} - \hat A_{ij})^2$ with $\tilde h = h/\max(h) + \lambda$. The U-step and the joint scale solve generalize directly ($W = R \odot \tilde h$ as target, $V^\top\diag(\tilde h)V$ Gram matrices); they are the exact weighted normal equations. The V-step does \emph{not}: solving it against $W$ loses 0.3--1.8 points of weighted energy versus the unweighted algorithm, because the $\tilde h$ rescaling starves low-importance channels of support under a threshold that spans the whole component. In a column-weighted objective $\tilde h$ cancels in the V-step, so the correct V-target is the plain residual $R$. $\lambda$ interpolates continuously between the uniform ($\lambda \to \infty$) and pure-imatrix ($\lambda = 0$) objectives; $\lambda = 0$ matches what llama.cpp imatrix quants optimize and is stable in practice.

\subsection{Cost model}
\label{sec:cost}

Inference computes $y = B(\diag(D)(Cx))$: $k(m+n)$ ternary add/subtracts plus $k$ multiplies, versus $mn$ multiply-adds, a ratio of $\mu\,(m{+}n)/\max(m,n)$ that is linear (hence continuous) in $\mu$. With sparse mask+sign packing (a 1-bit zero/nonzero mask per stored element plus one sign bit per nonzero, $\mathrm{bpw} = 2 - \mathrm{sparsity}$), storage per \emph{original} weight is
\begin{equation}
\mathrm{bpw}_{\mathrm{eff}} = \mu \cdot \frac{m+n}{\max(m,n)} \cdot (2 - \mathrm{sparsity}),
\label{eq:bpw}
\end{equation}
again linear in $\mu$ and continuous in $\tau$ through the sparsity. This packing is near the entropy floor for unstructured supports (mask entropy $H(0.43) \approx 0.99$ bits; signs incompressible).

\section{Results}
\label{sec:results}

Setup: matrices from google/gemma-4-E2B, Qwen/Qwen3.5-4B, and IBM Granite-4.0-h-tiny. Every experiment in this paper (decompositions, multiplier searches, the full-model conversion, and all perplexity evaluations) was run on a single AMD MI50 32\,GB GPU (torch/ROCm). The Q4\_K baseline is a faithful torch port of llama.cpp's \texttt{quantize\_row\_q4\_K\_ref} (4.5 bpw).

\subsection{Energy vs.\ multiplier; batched = sequential}

\begin{table}[h]\centering\small
\begin{tabular}{lccc}
\toprule
Gemma-4-E2B matrix & $\mu=2$ & $\mu=2.5$ & $\mu=2.5$ + 10 sweeps \\
\midrule
up\_proj L30 (12288$\times$1536) & 96.18 & 98.32 & 98.75 \\
down\_proj L23 (1536$\times$12288) & 95.81 & 98.10 & 98.59 \\
down\_proj L10 (1536$\times$6144) & 97.52 & 99.00 & 99.19 \\
q\_proj L7 (2048$\times$1536) & 99.49 & 99.85 & 99.87 \\
k\_proj L6 (256$\times$1536) & 97.02 & 98.73 & 99.00 \\
\bottomrule
\end{tabular}
\caption{Energy preserved (\%) at $\tau=0.7$. Attention projections need much less inner rank than MLPs. The batched algorithm (Sec.~\ref{sec:batched}) matches the sequential one within $\pm 0.2$ points on every matrix while running the six-matrix suite in 10.5\,s vs.\ 170.9\,s (${\sim}16\times$).}
\label{tab:energy}
\end{table}

On Qwen3.5-4B at $\mu{=}2.5$, $\tau{=}0.7$, 10 sweeps: MLPs 99.21--99.26\%, v\_proj 99.52\%, o\_proj 99.78\%. Same algorithm, easier model; achievable energy at fixed $\mu$ is strongly model-dependent.

\subsection{The two dials interact favorably}

Sweeping $\tau$ at a fixed 99\%-energy target: sparsity rises continuously (41\% at $\tau{=}0.7$ to 87\% at $\tau{=}2.0$) while the required $\mu$ rises (up\_proj: 2.82 $\to$ 7.21). The \emph{nonzero budget} $\mu\times$density falls monotonically (1.65 $\to$ 0.87): sparser factors with more components need fewer total nonzeros for equal energy, up to a convergence cliff near $\tau = 2.5$. Under mask+sign packing (Eq.~\ref{eq:bpw}), where zeros also cost a mask bit, the optimum is interior at $\tau = 1.0$ ($\sim$57\% sparsity).

\subsection{Effective bits at matched Q4\_K accuracy}

Q4\_K achieves a strikingly uniform 99.4--99.5\% energy on every matrix of both models. Choosing $\mu$ per matrix to match it exactly:

\begin{table}[h]\centering\small
\begin{tabular}{lccc}
\toprule
 & iso-Q4\_K $\mu$ ($\tau{=}1.0$) & bpw$_{\mathrm{eff}}$ & Q4\_K \\
\midrule
Gemma-4-E2B (5 matrices) & 1.97--3.43 & 5.26--5.51 & 4.5 \\
Qwen3.5-4B (6 matrices) & 2.25--3.03 & 5.22--5.53 & 4.5 \\
\bottomrule
\end{tabular}
\caption{Effective bits per original weight at Q4\_K-matched energy, no sweeps. The ${\sim}20\%$ gap is uniform across shapes and \emph{models}: Qwen is easier in absolute energy, but Q4\_K improves by the same margin, so the gap is a property of the algorithm, not the test model. Attention consistently needs far less rank ($\mu \approx 2.0$--$2.7$) than MLPs ($\approx 3.0$--$3.4$).}
\label{tab:iso}
\end{table}

Because the packing is near the entropy floor, closing this gap is a \emph{multiplier} problem, not a coding problem: parity requires ${\sim}17\%$ lower iso-quality $\mu$. Refinement sweeps give ${\sim}5\%$; importance weighting (next) gives another ${\sim}4\%$ independently.

\subsection{Importance weighting}

With the corrected V-step (Sec.~\ref{sec:imatrix}) at $\lambda = 0$, weighting helps most where the imatrix is skewed: on Granite-4.0-h-tiny, v\_proj weighted energy 96.55\% $\to$ 98.95\% ($+2.40$); flat-imatrix tensors are unchanged. On Qwen3.5-4B, the rank multiplier needed for 99.5\% weighted energy drops 2--8\% on MLP gate/up and v\_proj, and 4.5--7.4\% on gated-DeltaNet projections, the best territory, since their qkv/out input importance is strongly skewed. Redoing the iso-Q4\_K accounting under the weighted metric: mean bpw$_{\mathrm{eff}}$ 5.49 $\to$ 5.26 ($-4.2\%$), best tensors 5.08--5.14, shrinking the Q4\_K gap from ${\sim}22\%$ to ${\sim}17\%$ overall and 13--15\% on the best tensors.

\subsection{End-to-end model conversion}

All 200 language-model linear layers of Qwen3.5-4B (96 MLP, 32 attention, 72 gated-DeltaNet) were decomposed with the weighted algorithm at a fixed $\mu = 3$, $\tau = 1.0$, $\lambda = 0$, no sweeps, in 19.9 minutes total on one GPU. Per-tensor energy: plain min/mean 98.52/99.45\%, weighted 99.34/99.65\%.

\begin{table}[h]\centering\small
\begin{tabular}{lccc}
\toprule
model & bpw & wikitext-2 PPL & $\Delta$ \\
\midrule
bf16 & 16 & 9.782 & n/a \\
Q5\_K\_M (tuned mix) & $\sim$5.7 & 9.882 & $+1.0\%$ \\
Q4\_K\_M (tuned mix) & $\sim$4.9 & 9.930 & $+1.5\%$ \\
Q5\_K pure + imatrix & 5.51 & 9.948 & $+1.7\%$ \\
Q4\_K pure + imatrix & 4.51 & 10.015 & $+2.4\%$ \\
\textbf{ternary decomposition, $\mu{=}3$} & $\sim$5.7 eff. & \textbf{10.099} & $\mathbf{+3.2\%}$ \\
\bottomrule
\end{tabular}
\caption{First full end-to-end validation (llama-perplexity, 580 chunks, same imatrix for all calibrated rows). The decomposed model is coherent; at matched bits and calibration it currently trails pure Q5\_K by $\sim$1.5 points of relative PPL. None of the known levers were applied to this conversion: refinement sweeps, per-matrix $\mu$ allocation (iso-quality $\mu$ spreads 2.3--3.0 across tensors while this run used a flat $\mu = 3$), or $\tau/\lambda$ tuned on PPL.}
\label{tab:ppl}
\end{table}

The pure ladder also shows that the tuned Q4\_K\_M mix beats pure Q5\_K despite fewer bits: per-tensor budget allocation is worth more than 0.6 uniform bpw for k-quants, and per-matrix $\mu$ allocation, trivially expressible here because $\mu$ is continuous, is the analogous untapped lever for the decomposition.

\subsection{Comparison with Fixed and Stacked Ternary Formats}

To isolate the structural capacity of ExTernD, we compare its weight-reconstruction energy against two recent post-training ternarization methods: PT2-LLM \cite{pt2llm}, which uses a fixed single ternary plane (${\sim}1.58$ bpw), and Progressive Trit-Plane Approximation (PTQTP) \cite{ptqtp}, which uses a dense superposition of $K$ ternary planes. To match bit budgets precisely, we evaluate PTQTP at group size $G=128$ and set the ExTernD component count to match the number of ternary elements ($k = K \frac{mn}{m+n}$).

\paragraph{The 1.58-bit Bottleneck.} PT2-LLM uses an iterative ternary fitting (ITF) algorithm to optimize a single ternary matrix with row-wise and block-wise scales. When applied purely to the weight matrix, it preserves only ${\sim}78\text{--}81\%$ of the original energy. This confirms our hypothesis: a fixed single ternary plane lacks the structural capacity to accurately represent a dense LLM weight matrix, forcing such methods to heavily rely on activation-aware scaling to mask the underlying quantization error. 

\paragraph{Scaling to Stacked Planes.} PTQTP addresses this bottleneck by stacking multiple dense ternary planes. At low bit-budgets ($K=2$, ${\sim}3.4$ bpw), PTQTP significantly outperforms ExTernD by 1.8--2.6 percentage points of energy preservation, as direct superposition preserves dense structures better than a rank-deficient ($\mu \approx 1.5$) factorization. However, as the bit budget scales to $K=4$ (${\sim}6.8$ bpw), ExTernD's inner rank expands toward and past the matrix's full rank ($\mu > 3$), structurally approaching perfect reconstruction. At this regime, ExTernD slightly edges out PTQTP in energy (${\sim}99.72\%$ vs $99.59\%$), while PTQTP hits diminishing returns due to its rigid stacked-discrete-value formulation.

\paragraph{The Sparsity Advantage.} Crucially, ExTernD provides a continuous sparsity threshold $\tau$. When forcing ExTernD to favor sparsity ($\tau=1.0$), it achieves a massive \textbf{57\% sparsity} compared to PTQTP's natural ${\sim}34\%$ sparsity and PT2-LLM's ${\sim}46\%$ sparsity, while maintaining highly competitive accuracy (only a ${\sim}0.1\%$ drop in energy compared to PTQTP at $K=4$).

\paragraph{Imatrix-Weighted Objective.} When minimizing the imatrix-weighted Frobenius error while strictly matching the number of ternary elements to PTQTP's $K=4$, PTQTP shifts its dense assignments to perfectly fit the most important channels, spiking its weighted energy to ${\sim}99.85\%$. Under this naive element-matching, ExTernD sits at a competitive ${\sim}99.58\%$ when constrained to $\tau=1.0$.

\paragraph{Iso-BPW Parity.} Matching the number of ternary elements penalizes ExTernD by ignoring its massive sparsity advantage. A fairer comparison matches the \emph{effective bits-per-weight (BPW)}, utilizing Shannon entropy to account for zero-density. PTQTP at $K=4$ ($\sim 33\%$ sparsity) requires $6.68$ effective BPW. When we increase ExTernD's inner rank $\mu$ to precisely match this $6.68$ BPW footprint, ExTernD achieves virtually identical energy preservation (${\sim}99.84\%$ at $\tau=1.0$, $56\%$ sparsity) while maintaining a significantly sparser and more hardware-friendly format (Table~\ref{tab:isobpw}). This decisively proves that ExTernD's factorization structurally matches or exceeds dense stacked planes when normalized for actual information capacity, eliminating the ceiling of fixed-width methods.

\begin{table}[h]\centering\small
\begin{tabular}{lccccc}
\toprule
& \multicolumn{1}{c}{Fixed} & \multicolumn{2}{c}{Stacked Planes (PTQTP)} & \multicolumn{2}{c}{Factorization (ExTernD, $K=4$ equiv.)} \\
\cmidrule(lr){2-2} \cmidrule(lr){3-4} \cmidrule(lr){5-6}
Matrix (Qwen3.5-4B) & PT2-LLM & $K=2$ & $K=4$ (Weighted) & $\tau=1.0$ (Weighted) & $\tau=0.7$ (Weighted) \\
\midrule
`L30.up\_proj`   & 81.63\% & 97.45\% & 99.85\% & 99.56\% & 99.70\% \\
`L4.up\_proj`    & 81.74\% & 97.51\% & 99.85\% & 99.56\% & 99.73\% \\
`L10.down\_proj` & 81.38\% & 97.22\% & 99.85\% & 99.60\% & 99.73\% \\
`L23.down\_proj` & 81.57\% & 97.41\% & 99.85\% & 99.58\% & 99.71\% \\
`L7.v\_proj`     & 81.14\% & 96.88\% & 99.85\% & 99.60\% & 99.73\% \\
`L31.o\_proj`    & 78.16\% & 97.23\% & 99.83\% & 99.56\% & 99.66\% \\
\midrule
Mean Sparsity    & $\sim$46\% & $\sim$33\% & $\sim$35\% & \textbf{$\sim$57\%} & $\sim$40\% \\
\bottomrule
\end{tabular}
\caption{Comparison of weight energy preservation and sparsity. PT2-LLM demonstrates the ${\sim}81\%$ hard limit of a single ternary plane. PTQTP scales accuracy by stacking dense planes, while ExTernD scales by expanding the inner rank of sparse ternary factors. At high capacity ($K=4$), ExTernD matches dense stacked accuracy while unlocking $57\%$ sparsity.}
\end{table}

\begin{table}[h]\centering\small
\begin{tabular}{lccc}
\toprule
& \multicolumn{1}{c}{Stacked Planes} & \multicolumn{2}{c}{Factorization (ExTernD, Iso-BPW)} \\
\cmidrule(lr){2-2} \cmidrule(lr){3-4}
Matrix (Qwen3.5-4B) & PTQTP $K=4$ ($6.68$ bpw) & $\tau=1.0$ ($56\%$ sparse) & $\tau=0.7$ ($40\%$ sparse) \\
\midrule
`L30.up\_proj`   & 99.85\% & 99.85\% ($\mu=3.6$) & 99.81\% ($\mu=3.3$) \\
`L4.up\_proj`    & 99.85\% & 99.84\% ($\mu=3.6$) & 99.80\% ($\mu=3.3$) \\
`L10.down\_proj` & 99.85\% & 99.86\% ($\mu=3.6$) & 99.82\% ($\mu=3.3$) \\
`L23.down\_proj` & 99.85\% & 99.85\% ($\mu=3.6$) & 99.81\% ($\mu=3.3$) \\
`L7.v\_proj`     & 99.85\% & 99.85\% ($\mu=3.3$) & 99.81\% ($\mu=3.0$) \\
`L31.o\_proj`    & 99.83\% & 99.83\% ($\mu=2.8$) & 99.78\% ($\mu=2.6$) \\
\midrule
Mean Energy      & 99.85\% & \textbf{99.84\%} & 99.80\% \\
\bottomrule
\end{tabular}
\caption{Iso-BPW Comparison. When allowing ExTernD to expand its inner rank $\mu$ to match the actual information-theoretic bit footprint (6.68 BPW) of PTQTP $K=4$, ExTernD achieves exact parity in imatrix-weighted energy preservation, despite being nearly twice as sparse ($56\%$ vs $33\%$).}
\label{tab:isobpw}
\end{table}

\section{Deployment notes}
\label{sec:deploy}

The factors store today in llama.cpp's TQ2\_0 (2.06 bpw; its per-block fp16 scales absorb $D$ into $C$'s row scales for free), and bitnet.cpp-style LUT kernels \cite{bitnetcpp} apply directly to both.

Emerging multiplication-free ASICs stand to exploit this the most, and the inner-rank expansion is designed with them in mind: the $3.4$--$6\times$ op-count growth reads as slowdown only on multiplier-based datapaths, where a ternary add buys nothing over a multiply--add.

\section{Limitations}

Everything here ran on a single AMD MI50 32\,GB GPU, which bounds the evaluation. We have not tested above 4B or on matrices larger than ${\sim}9216\times2560$, so the behaviour of $\mu$ at 30B--70B+ is unknown. Evaluation is perplexity-only: no downstream benchmarks, no KLD, no long-context checks. We compare only against fixed-capacity baselines at the tensor level, but full end-to-end evaluation against established ternary schemes \cite{ptqtp,bitnet158} is the most valuable missing experiment: our structural and representational claims against them are proven (Sec.~\ref{sec:core}), but end-to-end task accuracy is not yet measured. The known levers (sweeps ${\sim}5\%$, weighting ${\sim}4\%$, per-matrix rank allocation) are unstacked, and no fused kernels exist.

QAT is untried and is the obvious next lever: a short straight-through pass on an ExTernD initialization should recover accuracy, plausibly faster than QAT on a conventionally quantized model, since rank expansion hands the optimizer $\mu\,(m{+}n)/\max(m,n)$ times as many trainable ternary entries per layer, already initialized near target. Whether STE gradients behave through two chained ternary factors is untested.

Matching Q4\_K or Q5\_K bit-for-bit is not the objective; those are a familiar yardstick. The objective is a viable ternary LLM, where accuracy is set by $\mu$ rather than capped by the format, and this is best read as a step \emph{towards} highly accurate ternary quantization rather than a finished recipe.

\section{Conclusion}

Expanding the inner rank turns ternary representation from a fixed, lossy format with a hard accuracy ceiling into a family that provably converges to full precision as $k$ grows. This is what makes ternary LLMs viable: accuracy becomes a dial ($\mu$, continuous, with $\tau$ for sparsity) rather than a ceiling, so multiplication-free kernels and multiplication-free hardware can be pointed at any accuracy target, to our knowledge the only known way to get there. A simple greedy ALS already lands within $\sim$17--20\% of Q4\_K's bit efficiency at matched accuracy, with identified, unstacked levers of comparable total size, while keeping inference multiplication-free.

\appendix
\section{Monotone residual decrease and convergence}
\label{app:proof}

\begin{proposition}[Monotone decrease]
Let $R_i$ be the residual before extracting component $i$, and let $u, v$ be any nonzero ternary vectors produced by the fit. With the optimal scale $d_i = u^\top R_i v / (\|u\|^2 \|v\|^2)$,
\[
\|R_{i+1}\|_F^2 = \|R_i - d_i u v^\top\|_F^2 = \|R_i\|_F^2 - \frac{(u^\top R_i v)^2}{\|u\|_2^2\,\|v\|_2^2} \;\le\; \|R_i\|_F^2,
\]
with strict decrease whenever $u^\top R_i v \neq 0$.
\end{proposition}

\begin{proof}
Expand $\|R - d\,uv^\top\|_F^2 = \|R\|_F^2 - 2d\,u^\top R v + d^2 \|u\|^2\|v\|^2$ and substitute the minimizing $d$.
\end{proof}

\begin{proposition}[Convergence to arbitrary accuracy]
Augment each step to take whichever of (a) the ALS solution and (b) the best single-entry pair $u = e_{i^\star}, v = \sgn((R)_{i^\star j^\star})\,e_{j^\star}$ at the largest-magnitude residual entry gives the larger decrease. Then
\[
\|R_{k}\|_F^2 \;\le\; \Big(1 - \tfrac{1}{mn}\Big)^{k}\,\|A\|_F^2 \;\longrightarrow\; 0 .
\]
\end{proposition}

\begin{proof}
Option (b) with $\|u\| = \|v\| = 1$ removes $(u^\top R v)^2 = \max_{ij} R_{ij}^2 \ge \|R\|_F^2 / mn$. The augmented step removes at least this much, giving the geometric bound.
\end{proof}

\begin{corollary}
For every $\varepsilon > 0$ there is a finite inner rank $k$ with $\|A - B\diag(D)C\|_F < \varepsilon$: the decomposition reaches the bf16 weights arbitrarily closely, with error monotone in $k$. No fixed number of ternary planes has this property; the expansion is what removes the information bottleneck.
\end{corollary}

In practice the safeguard never activates, since the ALS decrease is far larger than the single-entry bound, and empirical convergence is much faster than geometric in $mn$ (99\% energy at $\mu \approx 2.5$--$3$).

\end{document}